\documentclass[sigconf]{acmart}

\AtBeginDocument{%
  }

\copyrightyear{2023}
\acmYear{2023}
\setcopyright{rightsretained}
\acmConference[AIES '23]{AAAI/ACM Conference on AI, Ethics, and
Society}{August 8--10, 2023}{Montréal, QC, Canada}
\acmBooktitle{AAAI/ACM Conference on AI, Ethics, and Society (AIES '23), August
8--10, 2023, Montréal, QC, Canada}
\acmDOI{10.1145/3600211.3604656}
\acmISBN{979-8-4007-0231-0/23/08}

\usepackage{amsmath}
\usepackage{graphicx}
\usepackage{natbib}
\usepackage{bbm}
\usepackage{enumitem}
\usepackage{booktabs}
\usepackage{multirow}
\usepackage[ruled, linesnumbered]{algorithm2e}
\usepackage{amsthm}

\newtheorem{theorem}{Theorem}

\let\vec\mathbf

\DeclareMathOperator*{\argmin}{arg\,min}
\DeclareMathOperator*{\sgn}{sgn}
\newcommand{\aoc}{\ensuremath{A_{\mathrm{OC}}}}
\newcommand{\aga}{\ensuremath{A_{\mathrm{GA}}}}
\newcommand{\ars}{\ensuremath{A_{\mathrm{RS}}}}
\newcommand{\diceml}{\texttt{DiCE-ML}}

\setlist{parsep=0pt, itemsep=1.5pt, topsep=-2pt}

\begin{document}

\title{Iterative Partial Fulfillment of Counterfactual Explanations: Benefits and Risks}

\author{Yilun Zhou}
\affiliation{%
  \institution{MIT CSAIL}
  \city{Cambridge}
  \state{MA}
  \country{USA}
}
\email{yilun@csail.mit.edu}

\renewcommand{\shortauthors}{Yilun Zhou}

\begin{abstract}
  Counterfactual (CF) explanations, also known as contrastive explanations and algorithmic recourses, are popular for explaining machine learning models in high-stakes domains. For a subject that receives a negative model prediction (e.g., mortgage application denial), the CF explanations are similar instances but with positive predictions, which informs the subject of ways to improve. While their various properties have been studied, such as validity and stability, we contribute a novel one: their behaviors under \textit{iterative partial fulfillment} (IPF). Specifically, upon receiving a CF explanation, the subject may only partially fulfill it before requesting a new prediction with a new explanation, and repeat until the prediction is positive. Such partial fulfillment could be due to the subject's limited capability (e.g., can only pay down two out of four credit card accounts at this moment) or an attempt to take the chance (e.g., betting that a monthly salary increase of \$800 is enough even though \$1,000 is recommended). Does such iterative partial fulfillment increase or decrease the total cost of improvement incurred by the subject? We mathematically formalize IPF and demonstrate, both theoretically and empirically, that different CF algorithms exhibit vastly different behaviors under IPF. We discuss implications of our observations, advocate for this factor to be carefully considered in the development and study of CF algorithms, and give several directions for future work. 
\end{abstract}




\begin{teaserfigure}
  \vspace{-0.13in}
  \includegraphics[width=\textwidth]{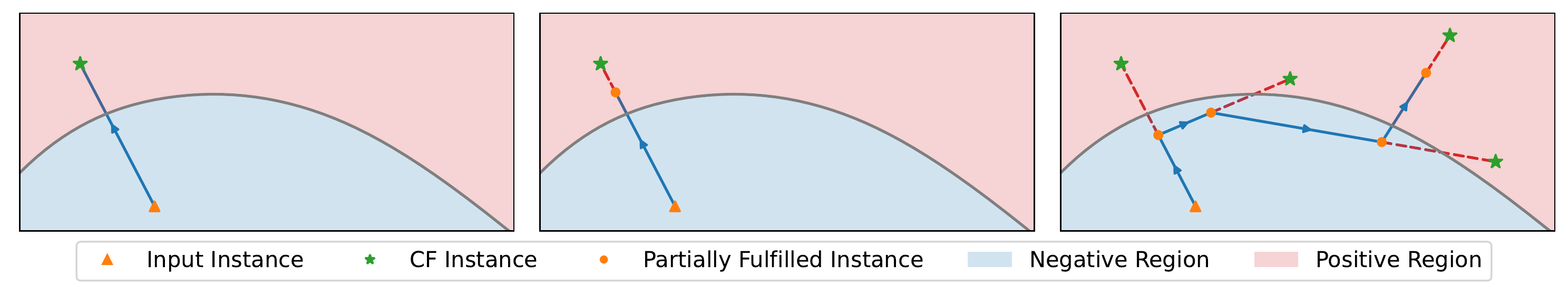}
  % \vspace{0.0in}
  \vspace{-0.25in}
  \caption{Left: a counterfactual (CF) explanation (green star) of an negative input instance (orange triangle) is often assumed to be fulfilled (i.e., achieved) completely, thus guaranteeing a positive subsequent prediction. In this paper, we consider \textit{partial fulfillment}, where the subject, by choice or by necessity, stops in the middle (orange circle). Middle: if the CF explanation is far from the decision boundary, then a partial fulfillment near the goal could be sufficient for a positive prediction while saving improvement cost. Right: however, CF algorithms often use local gradient ascent and/or randomized search, which could result in a long and winding path towards the final positive prediction and much higher total cost. }
  \label{fig:1}
  \vspace{0.06in}
\end{teaserfigure}

\maketitle

\setlength{\parskip}{0pt plus 0.1pt minus 0.1pt}

\section{Introduction}

Recently, machine learning models have been increasingly deployed in high-stakes domains in finance, law and medicine, performing tasks such as loan approval \citep{german1994}, recidivism prediction \citep{larson2016we} and medical diagnosis \citep{johnson2016mimic}. For these domains, the reason why a particular prediction is made is often as important as the prediction itself, especially since most of the high performing models, such as neural networks and random forests, are black-box in nature. In some jurisdictions, the ``right to explanation'' is even legally required for people receiving adverse model predictions (e.g., mortgage application denial) to understand the reason and available recourses. 

For these purposes, counterfactual (CF) explanations, also known in the literature as contrastive explanations or algorithmic recourses, have been a popular choice due to their desirable properties in human psychology and cognition theories \citep{miller2019explanation}. For a particular input $x$ with a certain model prediction $\tilde y$, its CF explanation is another input similar to $x$ but with a prediction $\tilde y'$ different from $\tilde y$. Thus, this explanation indicates how the input would need to change in order for the model prediction to also change, as shown in Fig.~\ref{fig:1} (left). When $\tilde y$ is a negative prediction (e.g., mortgage application denial) and $\tilde y'$ is a positive one (e.g., application approval), this CF explanation essentially gives as a direction for the applicant to improve their situation and get the application approved the next time, given that it is feasible (e.g., not changing immutable features such as gender and race), which automatically avoids the unfaithfulness problem of many feature attribution explanations \citep{ribeiro2016should, lundberg2017unified, simonyan2013deep} where the salient features are not really important for the model's decision making \citep{yang2019benchmarking, zhou2021feature, adebayo2021post}. 

Over the past few years, there have been many investigations into various properties of these counterfactual explanations, such as their validity, action feasibility and cost, stability with respect to input perturbations and model updates, and agreement with the underlying causal mechanism, which are reviewed in Sec.~\ref{sec:related-work}. Taken as a whole, they generate quality profiles for various CF explanation algorithms and establish their relative strengths and weaknesses. 

In this paper, we propose a novel aspect of these explanations, which, to the best of our knowledge, has not been studied before. Specifically, in real life scenarios, the subject of the prediction (e.g., the mortgage applicant) may not completely fulfill the CF explanation for various reasons. First, the subject may not be able to do so. For example, the CF explanation requires the subject to pay down all four credit card accounts, but they can only pay down two of them with their current saving. Second, the subject may want to take their chance and get a favorable outcome with less effort. For example, the CF explanation recommends the subject to increase their monthly salary by \$1,000 but they wonder if an increase of just \$800 would be sufficient. Last, the subject may misunderstand the wording of a CF explanation (e.g., a bullet list in the notice of denial) as taking any \textit{subset} of actions rather than all of them. In all cases, we have a \textit{partial} fulfillment of the CF explanation followed by a re-query for the model prediction and CF explanation, repeating until a positive prediction is obtained. We term this process as the \textit{iterative partial fulfillment} (IPF) of CF explanations. 

How does IPF, compared to a one-shot complete fulfillment, affect the subject welfare? In particular, does the subject need to make more total improvement under IPF? Intuitively, the effect can be positive, negative or neutral. 

A positive effect could result from CF algorithms that generate instances landing well into the positive model prediction region. A less complete fulfillment (e.g., a salary increase of \$800 rather than the recommended \$1,000) is still sufficient, allowing the subject to incur a lower cost of improvement, as shown in Fig.~\ref{fig:1} (middle). 

By contrast, if the initial partial fulfillment (e.g., paying down two out of the four recommended credit card accounts) is unsuccessful, the new CF explanation may suggest some other factors to change and effectively ``reset'' the progress, resulting in a larger total cost of improvement, as shown in Fig.~\ref{fig:1} (right). In the most extreme case, an oscillation may even occur among the series of explanations, leaving the subject stuck in an infinite loop. 

Last, if, for an input $x$ and its CF explanation $x'$, the CF explanations for all partially fulfilled input states along the trajectory of $x\rightarrow x'$ are also the same counterfactual explanation $x'$, and no partial fulfillment results in a positive prediction, then the total improvement cost under IPF is the same as that under one-shot complete fulfillment, since both will lead to the subject achieving $x'$ in the end. Such as scenario is possible if the CF algorithm works by either finding the exact closest input with a different target prediction or returning a CF instance from a small set of candidates. 

In practice, the specifics of a CF algorithm, such as the optimization method and considerations for stochasticity and diversity, determine the net effect of these three possibilities, making certain CF algorithms preferable to others from the IPF welfare perspective. 

In this paper, we investigate this problem by first formalizing the notion and implementation of IPF in Sec.~\ref{sec:formalization}. Then in Sec.~\ref{sec:theoretical} we theoretically prove that certain CF algorithms can exhibit positive, negative and neutral effects on IPF welfare (i.e., total improvement cost compared to the one-shot complete fulfillment), as conceptually explained above. Empirically, in Sec.~\ref{sec:empirical} using two financial datasets, three CF algorithms and four CF generation configurations per algorithm (for a total of 24 setups), with the widely used \diceml{} package\footnote{\url{https://github.com/interpretml/DiCE}}, we confirm that the generated explanations indeed possess different IPF characteristics. Thus, from these two pieces of evidence, we argue that an IPF assessment should be part of a comprehensive evaluation of CF algorithms. Finally, in Sec.~\ref{sec:discussion-conclusion}, we discuss the broader technical and societal implications of IPF and its future directions. 

\section{Related Work}
\label{sec:related-work}

In this section, we discuss works on the algorithms and properties of counterfactual explanations, with focus on impact, recency and relevance to IPF. For much more detailed discussions, we refer the readers to standalone surveys, such as those by \citet{byrne2019counterfactuals}, \citet{guidotti2022counterfactual}, \citet{keane2021if} and \citet{verma2020counterfactual}. 

CF explanations are popularized by \citet{wachter2017counterfactual}, who proposed a gradient ascent algorithm to search for a counterfactual instance that both achieves the target model prediction and is close to the original input being explained. Subsequent works have extended the basic idea to make CF explanations diverse \citep{mothilal2020dice, russell2019efficient}, in-distribution \citep{van2021interpretable, poyiadzi2020face}, aware of the causal mechanism \citep{karimi2021algorithmic}, and less susceptible to gaming \citep{chen2020linear}. In addition, different optimization methods have also been proposed, such as those that do not require model differentiability or gradient access \citep{mothilal2020dice, pawelczyk2020learning}, as well as those based on integer programming \citep{ustun2019actionable, russell2019efficient}, constraint satisfaction \citep{karimi2020model} and optimal transport \citep{de2021transport}. 

A recent line of work aims to generate a \textit{sequence} of instances from the input to the final CF explanation \citep{ramakrishnan2020synthesizing, kanamori2021ordered, naumann2021consequence, verma2022amortized}. This sequence provides an explicit path for the subject to follow, and is argued to be more user-friendly and actionable. While our proposed IPF setup also results in a sequence of explanations, it is fundamentally different in both goals and constraints. In sequential generation, the explanation algorithm has full control of the generated sequence and only the last sequence element needs to be the CF explanation (i.e., inducing the target prediction). By comparison, in IPF, the algorithm needs to work with whatever partially fulfilled instance provided by subject and generate a valid CF explanation. 

\citet{kenny2021generating} and \citet{aryal2023even} proposed to generate semi-factual explanations, defined as data instances which move towards the decision boundary but have not crossed it. These explanations could be used to construct ``even if'' explanations: e.g., \textit{even if the down payment is \$10,000 more, the mortgage would still not be approved}. A partially fulfilled CF in IPF and a semi-factual instance are both on the way to some CF instance, but they are otherwise distinct concepts -- one is an intermediate state of the subject and the other is an explanation. 

On the evaluation and analysis side, various properties of CF explanations have been proposed. The two core desiderata of CF explanations are validity and feasibility. The former is just the success rate of the CF generation, while the latter, conceptually defined as the ease for the subject to follow the CF recommendation, is much more nuanced. Different approaches have been proposed to enforce and evaluate it, such as ensuring a close distance to the original data point \citep{wachter2017counterfactual}, lying in a high-density region of the data distribution \citep{van2021interpretable}, satisfying causal constraints \citep{karimi2021algorithmic}, and respecting custom limitations in modifying feature values \citep{mothilal2020dice, ustun2019actionable}. 

Most relevant to our IPF proposal are two notions of stability. The first is with respect to input perturbations, where \citet{dominguez2022adversarial} and \citet{virgolin2023robustness} found that a given input can be adversarially but minimally perturbed into an instance with a very different CF explanations. \citet{maragno2023finding} proposed a robust optimization formulation to find stable CF explanations. \citet{slack2021counterfactual} demonstrates that a model could be trained to make this behavior more prevalent and hide discrimination issues. The second one is with respect to the model update, where a new model is trained on an updated version of the dataset. \citet{rawal2020algorithmic} found that many CF algorithms are often unstable under model update in that very different explanations are generated for the new model and the original ones are no longer valid, and \citet{upadhyay2021towards} proposed an algorithm to find CF instances that are stable under model update. 

From this perspective, our IPF property can be considered as a third notion of stability: the stability of CF explanations for inputs along the path of improvement. If the CF explanations are stable, then the subject will follow a mostly consistent path of improvement, while if not, the subject may be given unrelated or even contradicting recommendations after every partial fulfillment. 

\section{Iterative Partial Fulfillment}
\label{sec:formalization}
\subsection{Background}
\label{sec:background}

In this section, we formalize the concept of iterative partial fulfillment (IPF) of CF explanations. Due to the variety of real world human behaviors, there are many ways to formalize IPF. As we are the first to do so, we provide and analyze one canonical setup, and discuss other design choices and extensions in Sec.~\ref{sec:discussion-conclusion}. 

Let $\mathcal X$ be the input space with $D$ features, and $\mathcal X_d$ for $d\in \{1, ..., D\}$ be the set of values for the $d$-th feature. We consider categorical and numerical features, where $X_d$ is a finite set for the former and (a subset of) $\mathbb R$ for the latter. Thus, an $x\in\mathcal X$ can be written as $(x_1, ..., x_D)$ with $x_d\in \mathcal X_d$. For notational simplicity, we restrict ourselves to binary classification tasks, and represent the model prediction function as $m: X\rightarrow [0, 1]$ that returns the predicted probability of the positive class. Thus, $m(x)\geq 0.5$ indicates a positive prediction. In the ensuing discussion, we consider negatively predicted input instances $m(x) < 0$ and their positively predicted counterfactuals $m(x')\geq 0.5$. 

We denote a CF explanation algorithm as $A: \mathcal X\rightarrow \mathcal X$, which takes an input instance and returns another instance. Since some algorithms are stochastic, we allow $A(x)$ to return a random $x'$ sampled from the corresponding distribution. In addition, some algorithms generate a set of diverse CF explanations, and the subject chooses one of them as the goal using some strategy, such as picking the most similar one or selecting one uniformly at random. In this case, we let $A(\cdot)$ to manage this CF selection procedure so that it always returns a single (but possibly stochastic) CF explanation. To simplify notation, we write $x'=A(x)$ if $A$ is deterministic on $x$ and $x'\sim A(x)$ for a sampled value. 

For $x$ and $x'$, to represent the cost of change between the two, we define a cost metric 
\begin{align}
c(x, x')=\sum_{d=1}^D c_d(x_d, x_d'),
\end{align}
where $c_d(x_d, x_d')$ is the per-feature cost. If the feature is categorical, we have $c_d(x_d, x_d') \allowbreak = \mathbbm 1\{x_d\neq x_d'\}$. Otherwise, we have $c_d(x_d, x_d')=|F_d(x_d) - F(x_d')|$ where $F$ is cumulative distribution function of the feature values, following \citet{ustun2019actionable}, to account for the feature value density, with the maximal change incurring a cost of 1. Given a sequence of $N$ instances $\vec x = (x^{(1)}, ..., x^{(N)})$, the total cost for this sequence is the sum of pairwise neighbor costs $c(\vec x) = \sum_{n=1}^{N-1} c(x^{(n)}, x^{(n+1)})$. 

\subsection{Partial Fulfillment}
We now formalize the partial fulfillment as follows. 

\begin{definition}[$u$-partial fulfillment]
For current state $x$ and goal state $x'$ (e.g., as generated by the CF algorithm), the $u$-partial fulfillment $w\in\mathcal X$, with $u\in [0, 1]$, is generated by the following operation on each feature: 
\begin{itemize}[leftmargin=3.5ex]
    \item if the feature is continuous, the new feature value $w_d$ is an interpolation between the two feature values: $(1-u)\cdot x_d + u\cdot x_d'$, except that when $|x_d-x'_d|\leq \epsilon$, the new value is $x_d'$; 
    \item if the feature is categorical, the new feature value $w_d$ takes $x_d$ with probability $1-u$ and $x_d'$ with probability $u$. 
\end{itemize}
Since categorical feature values are generated stochastically, we use $\phi(x, x', u)$ to denote the distribution of partial fulfillment $w$. 
\end{definition}
Conceptually, from the subject's perspective, when partially fulfilling $x'$ from $x$, at an effort level $u$, for every continuous feature, they will move an amount proportional to $u$ towards the goal feature value, and for every categorical feature, they will choose to make the change with probability $u$. Thus, the partial fulfillment result is stochastic as long as there is at least one categorical feature value change required. A technical exception is put on continuous features, where the partially fulfilled value is set to the CF value if the value difference is small. This ensures the success of IPF when the CF instance lies exactly on the decision boundary. 

\begin{algorithm}[!t]
\textbf{Input}: initial input $x$, model $m$, CF explanation algorithm $A$, fulfillment effort level $u$, maximum number of iterations $T$\;
$t \gets 0$\;
$\vec x\gets [x]$\;
\While{$m(x) < 0.5$ and $t<T$}{
$x'\gets A(x)$\;
$x\gets \phi(x, x', u)$\;
$\vec x.\mathrm{append}(x)$\;
$t\gets t + 1$\;
}
\textbf{return} $\vec x$\; 
\caption{The iterative partial fulfillment (IPF) process.}
\label{alg:ipf}
\end{algorithm}

Given this partial fulfillment definition, we model the iterative partial fulfillment (IPF) process in Alg.~\ref{alg:ipf}. The subject starts with an input $x$, and repeatedly requests a counterfactual explanation to partially fulfill, until receiving a positive prediction or reaching a maximum number of iterations. The algorithm returns $\vec x$, a sequence of states that the subject has been. For effort level $u$, maximum number of iterations $T$, model $m$ and CF algorithm $A$, we use $\xi(x, u, T, m, A)$ to represent the distribution of realized state trajectories $\vec x$. When it is clear from the context, we omit some of the input arguments, such as $m$. The most direct measure of subject welfare under IPF is the total improvement cost $c(\vec x)$. Other metrics include final success rate and number of steps. If we are interested in the fairness implications of IPF (i.e., whether one demographic group is disproportionately affected by IPF), we can also compute these metrics separately for each group, as we conduct in Sec.~\ref{sec:empirical}. 

\section{Theoretical Analysis}
\label{sec:theoretical}
\subsection{IPF Stability}

Does IPF always increase or decrease the total improvement costs? As we demonstrate in this section, its effects on different CF algorithms are different. First, we formally define the concept of IPF stability discussed at the end of Sec.~\ref{sec:related-work}, which is a sufficient condition for cost preservation (i.e., IPF does not increase the total cost). 

\begin{definition}[IPF stable]
    A CF algorithm $A$ is IPF stable at $x$ if 
    \begin{enumerate}[leftmargin=4.5ex]
        \item $A$ is deterministic at $x$, and
        \item $\forall \, w\in \Phi(x, A(x))$, $A$ is deterministic at $w$ and $A(w)=A(x)$. 
    \end{enumerate}
    A CF algorithm $A$ is IPF stable globally if it is IPF stable at all $x\in\mathcal X$. 
\end{definition}

For IPF stable CF algorithms, we are assured that IPF never makes the total cost higher, compared to one-shot complete fulfillment. 

\begin{theorem}
If a CF algorithm $A$ is IPF stable at $x$, then for all $u$ and $T$, $\mathbb E_{\vec x\sim \xi(x, u, T, \allowbreak A)}[c(\vec x)] \allowbreak \leq c(x, A(x))$. 
\end{theorem}

The proof is straightforward. At every iteration of IPF, the same CF explanation is given. Thus, the total improvement cost is upper bounded by $c(x, A(x))$. If the model gives a positive prediction in some intermediate step (or $T$ is not large enough to achieve $A(x)$ or a positive prediction), the total improvement cost is strictly less, which could happen when $A$ is configured to be ``conservative'' and gives a CF instance of high model confidence. 

\subsection{Cost-Preserving/Decreasing CF Under IPF}

Do IPF stable CF algorithms exist? Obviously, a constant-valued CF algorithm that always produces the same CF instance $A(\cdot)=x'$ is stable, but this serves as a terrible CF explanation for most of the dissimilar input instances. More usefully, we show that the optimal cost CF algorithm is also stable. 

\begin{theorem}
For $p\geq 0.5$, the optimal cost CF algorithm 
\begin{align}
\aoc(x)=\argmin_{x': m(x')\geq p} c(x, x'),
\end{align}
which gives the instance closest to $x$ with model prediction at least $p$ (using deterministic tie-breaking if necessary), is IPF stable globally. 
\end{theorem}

\begin{proof}
We first consider the case of all numerical features and no categorical features. Recognizing the feature-wise absolute value CDF distance function $c_d=|F_d(x_d') - F_d(x_d)|$, we define sign flag $s_d=\sgn(x_d' - x_d)$, and have 
\begin{align}
    c(x, x')=\sum_{d=1}^D s_d(F_d(x_d') - F_d(x_d)).
\end{align}
Therefore, the search over the best CF can be reduced to that in $2^D$ ``quadrants,'' with one value of $s=(s_1, ..., s_D)$ specifying one quadrant $Q_s$. Denote the globally optimal CF as $x^*$. We need to show that IPF preserves the optimality of $x^*$ within the quadrant and across different quadrants. 

For within-quadrant optimality, without loss of generality, suppose that $x^*$ lives in $Q_s$ with $s=(1, ..., 1)$. Consider the new state $w=(1-u)x+ux^*$ resulting from the partial fulfillment. For all $x'\in Q_s$, 
\begin{align}
    &c(x^*, x)\leq c(x', x)\\
    \implies&\sum_d F_d(x^*) - F_d(x) \leq \sum_d F_d(x') - F_d(x) \\
    \implies &\sum_d F_d(x^*) \leq \sum_d F_d(x')\\
    \implies &\sum_d F_d(x^*) - F_d(w) \leq \sum_d F_d(x') - F_d(w)\\
    \implies &\sum_d c_d(x^*, w) \leq \sum_d c_d(x', w)\\
    \implies &c(x^*, w) \leq c(x', w). 
\end{align}
The last line establishes the within-quadrant optimality of $x^*$ for $w$. For across-quadrant optimality, consider the $s^*$ for the quadrant of $x^*$ and $s'$ for that of an CF instance $x'$ in a different quadrant. For a feature $d$ such that $s^*_d=s'_d$, $w_d$ makes the same amount of improvement towards both CFs (except when $w$ overshoots with respect to $x'_d$, which offsets the improvement on $x'_d$), while for $d$ such that $s^*_d\neq s'_d$ (which must exist because $s^*\neq s'$), the improvement towards $x^*_d$ strictly makes the distance to $x_d'$ worse. Thus, if $x^*$ is optimal for $x$ across all quadrants, it is still optimal for $w$. 

Combining within-quadrant and across-quadrant optimality together, we see that IPF preserves the optimality of $\aoc(x)$ (for inputs of all numerical features). 

For a categorical feature $d$ that needs change (i.e., $x^*_d\neq x_d$), if $w_d=x_d$, it does not change the feature-wise cost $c_d$ for any target instance (including $x^*$), while if $w_d=x^*_d$, it reduces the cost for $x^*$ by 1, and it reduces that for any other $x'$ by at most 1, if $x'_d=x^*_d$). Thus, the cost reduction for $x^*$ is as fast as any other $x'$, so adding the cost for categorical features to the overall cost $c$ does not affect the optimality of $x^*$. This completes the proof. 
\end{proof}

With similar proofs, the theorem also applies to $l_1$ distance functions, including the case of different features scaled differently (e.g., by respective mean absolute deviation as used by \citet{wachter2017counterfactual}), or with $F_d$ being arbitrary monotonic functions. These variations greatly increases the generality of the theorem. 

This algorithm is considered as the gold standard by many works that propose approximate procedures due to the intractability of the exact optimization, such as local gradient ascent \citep{wachter2017counterfactual} or randomized search \citep{mothilal2020dice, pawelczyk2020learning}. Hence, we see that IPF is not a concern in the ideal case. In fact, for conservative \aoc{} with $p > 0.5$, it is likely that the total cost of IPF is smaller due to early stopping. 

Moreover, the result extends easily to look-up based CF algorithms, as defined below. 
\begin{theorem}
A look-up based CF algorithm
\begin{align}
    A_{\mathrm{LU}}(x)=\argmin_{x'\in S} c(x, x'), 
\end{align}
which selects the instance closest to $x$ from a (finite) set of candidates $S$ (using deterministic tie-breaking if necessary), is IPF stable globally. 
\end{theorem}

The proof is analogous. A natural choice of $S$ (for a negatively predicted instance $x$) is the set of correctly predicted positive training instances. Indeed, using the training set as a constraint or regularization is a common ingredient in many CF algorithms \citep{van2021interpretable, poyiadzi2020face}, often to make the CF explanations more realistic and thus feasible, while this theorem demonstrates an added benefit of it. 

Putting everything together, we reiterate the central results of this section with the following corollary: 
\begin{corollary}
Both optimal cost and look-up based CF algorithms (\aoc{} and $A_{\mathrm{LU}}$) are IPF stable. 
\end{corollary}

\subsection{Cost-Increasing CF Under IPF}

Next, we demonstrate that two popular approximation methods, gradient ascent and randomized search, are prone to increasing the total improvement cost, possibly without limit. 

For differentiable models, gradient ascent is often used from the current input to find a CF instance that offers a good trade-off between the model prediction and distance, sometimes with other considerations. Different works have proposed different objective functions, with the earliest one proposed by \citet{wachter2017counterfactual} as
\begin{align}
    g(x') = \lambda (1 - m(x'))^2 + d_{\mathrm{MAD}}(x, x'), \label{eq:wachter}
\end{align}
where $d_{\mathrm{MAD}}$ is the $l_1$ distance weighted by the inverse median absolute deviation (MAD) per feature, and $\lambda$ controls the trade-off. We define a gradient ascent CF function \aga{} as the one that follows the gradient of $g(\cdot)$ from $x$ to the local minimum or the boundary of $\mathcal X$. If this end state does not achieve the required model prediction $p$, we return a default positive instance (which can be a fixed correctly classified positive training instance). 

It turns out that \aga{} could lead to arbitrarily bad IPF behaviors due to an oscillation phenomenon. 
\begin{theorem}
There exists a model $m$, input instances $x^{(1)}, x^{(2)}$ with all continuous features, and effort level $u$, such that $\phi(x^{(1)}, \allowbreak \aga{}(x^{(1)}), \allowbreak u)=x^{(2)}$ and $\phi(x^{(2)}, \aga{}(x^{(2)}), u)=x^{(1)}$. 
\end{theorem}

In this case, starting at $x^{(1)}$ and making partial fulfillment with effort level of $u$ results in an oscillation of $x^{(1)}\rightarrow x^{(2)}\rightarrow x^{(1)}\rightarrow x^{(2)}\rightarrow ...$ for $T$ steps. A concrete example is illustrated in Fig.~\ref{fig:theoretical-gradient}, which plots gradient field of the 2-dimensional objective function $g(\cdot)$ as gray arrows pointing in the \textit{ascent} direction. A ``valley'' (blue dashed line) separates the inputs into two regions. We have two instances, represented by orange and green circles. For each instance, the gradient ascent yields the red trajectory to the star marker of the same color. However, starting from the orange circle, a 0.5-partial fulfillment towards the orange star lands exactly on the green circle, whose counterfactual explanation is the green star, but a 0.5-partial fulfillment goes back to the orange circle again. 

\begin{figure}[!t]
    \centering
    \includegraphics[width=0.95\columnwidth]{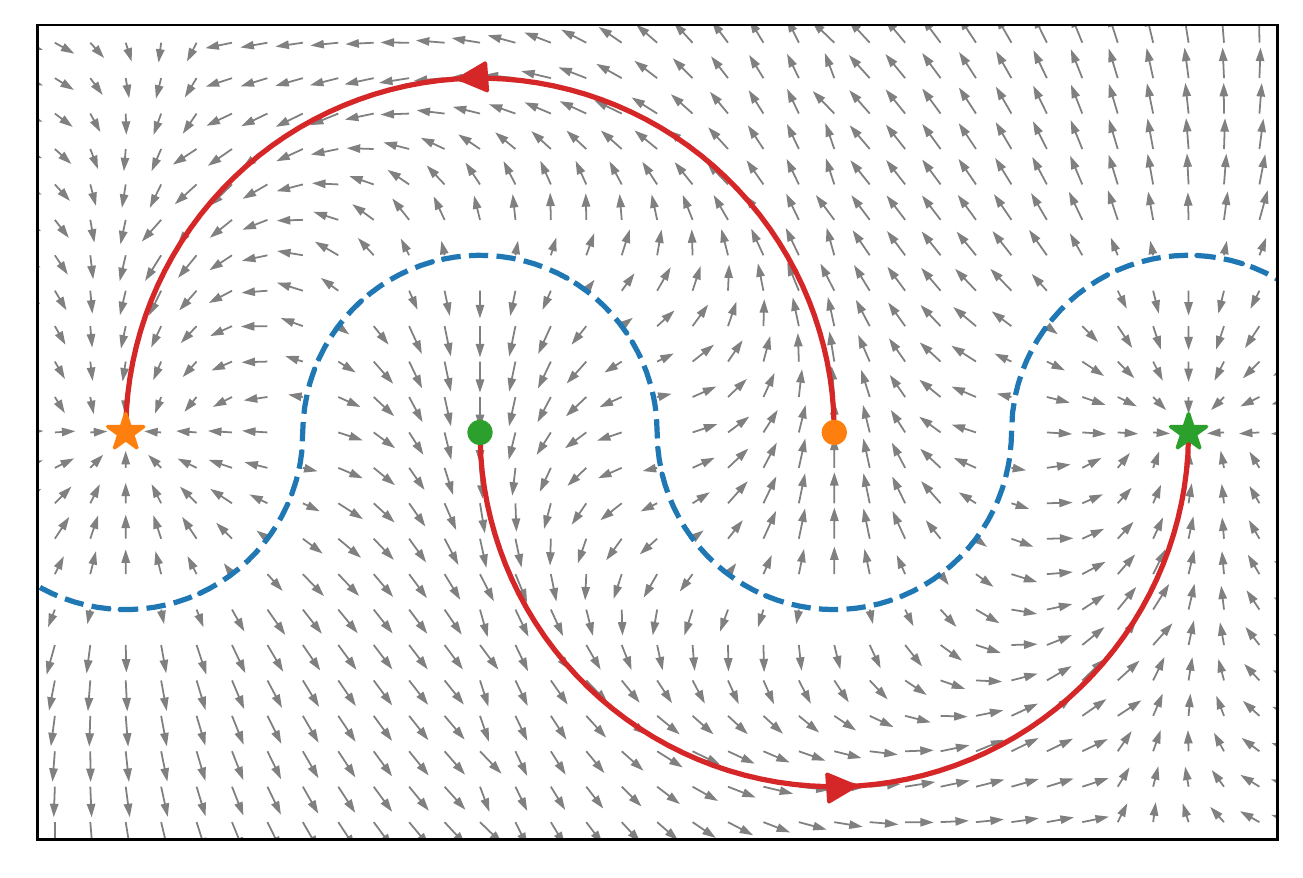}
    \caption{An example illustrating the oscillation behavior of gradient-ascent CF algorithms under partial fulfillment. }
    \label{fig:theoretical-gradient}
\end{figure}

The root cause for this issue is the non-optimality of the gradient ascent algorithm, in that it may only find farther local minima by following the gradient, such that a partial fulfillment (which move in the straight-line path, not along the gradient) could reset the progress. While the above example can be easily solved by caching the CF explanations found so far and returning the closest one if the gradient ascent cannot do better, models trained on real-world datasets with a large number of features may have many local minima in the high-dimensional input space, as evidenced by the prevalence of adversarial examples \citep{goodfellow2014explaining, liu2016delving}, rendering such caching effort mostly futile. 

A realistic example can be constructed as follows for the mortgage approval task. Consider two features, current saving $x_s$ and current debt $x_d$, where the model makes positive predictions on $\langle x_s, x_d\rangle=\langle \$10k, \$1k\rangle$ and $(\$40k, \$4k\rangle$. Now consider $x^{(\mathrm{CF},1)}=\langle \$20k, \$2k\rangle$ and $x^{(\mathrm{CF},2)}=\langle \$30k, \$3k\rangle$. Due to the shape of the objective function, the gradient ascent optimizes $\langle \$20k, \$2k\rangle$ to $\langle \$40k, \$4k\rangle$, and $\langle \$30k, \$3k\rangle$ to $\langle \$10k, \$1k\rangle$. As a result, starting at either $x^{(1)}$ or $x^{(2)}$ leads to an oscillation between the two.

Another popular approach, especially for non-differentiable models, is based on randomized search. Generally speaking, a randomized search algorithm \ars{} draw samples from the input space $\mathcal X$ using some strategy (e.g., uniformly at random or weighted towards the input instance $x$), and returns the best sampled instance according to some objective function (e.g., Eq.~\ref{eq:wachter}). However, this approach is also prone to increasing the total improvement cost under IPF. 

\begin{theorem}
    There exists a model $m$, an input instance $x$, and an effort level $u$, such that 
    \begin{align}
        \mathbb E_{\vec x\sim \xi(x, u, T, m, \ars{})}[c(\vec x)] > E_{x'\sim \ars(x)}[c(x, x')]
    \end{align}
\end{theorem}

Intuitively, this theorem should not be surprising: at step $t$ and state $x^{(t)}$, when a new CF goal $x^{(\mathrm{CF}, t+1)}$ is set, some of the effort expended during the previous round of partial fulfillment becomes wasted if the new goal requires a different fulfillment operation from the previous state $x^{(t-1)}$; i.e., $x^{(t)}\notin \Phi(x^{(t-1)}, x^{(\mathrm{CF}, t+1)})$. 

As a simple example, consider a probabilistic CF algorithm $A$ that gives one of two CF explanations, $x^{(\mathrm{CF}, 1)}$ and $x^{(\mathrm{CF}, 2)}$. For an input $x$, let $d_1$ and $d_2$ be the Euclidean distance to them respectively (assuming all continuous features). We have 
\begin{align}
    A(x) = \begin{cases}
x^{(\mathrm{CF}, 1)} \quad & \text{with (unnormalized) probability $d_1^{-1}$, }\\
x^{(\mathrm{CF}, 2)} \quad & \text{with (unnormalized) probability $d_2^{-1}$. }
    \end{cases}
\end{align}
Thus, if we have the initial state starting at the middle of these two CF states, $x=(x^{(\mathrm{CF}, 1)} + x^{(\mathrm{CF}, 2)}) / 2$, with $u=0.5$ (i.e., fulfilling halfway through the CF explanation), then the probability of CF always recommending the same counterfactual is 
\begin{align}
    \frac{3}{4} \cdot \frac{7}{8}\cdot \frac{15}{16}\cdot \frac{31}{32}... \approx 0.58, 
\end{align}
meaning that 42\% of times, there is at least one step that erases the effort of an earlier step. On our earlier mortgage approval example, these two states could represent the two ways of getting approved (high saving and high debt, or low saving and low debt), and a partially fulfilling applicant risks receiving contradictory feedback every time they make an application. 

\begin{figure}[!b]
    \centering
    % \vspace{0.4in}
    \includegraphics[width=0.8\columnwidth]{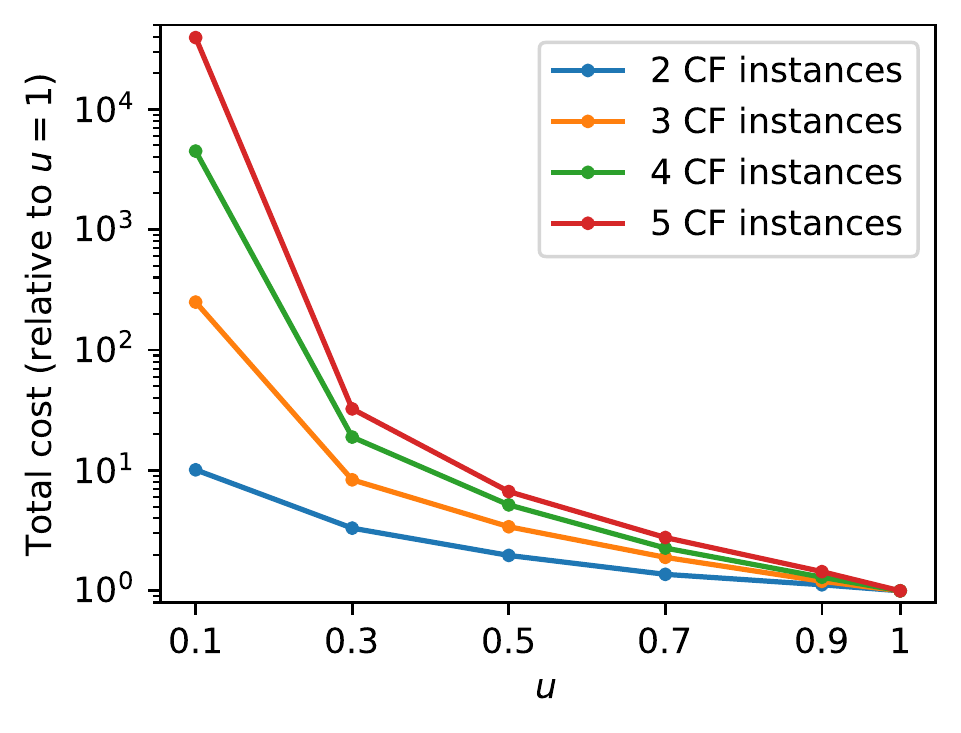}
    \caption{Total improvement cost as a multiple of the one-shot complete improvement for different effort levels $u$. }
    \label{fig:theoretical-random}
\end{figure}

Using a Monte Carlo simulation, Fig.~\ref{fig:theoretical-random} (blue line) shows the total improvement cost at different effort levels $u$ relative to that under single-shot complete fulfillment. Analogous results for the same setup but with three to five counterfactual states arranged on a regular polygon (with initial state $x$ at the center) are also presented in different colors. 

As we can see, a smaller value of $u$ and a larger number of candidate CF instances of all exacerbates the total improvement cost under IPF. In particular, with just five CF instances and an effort level of 0.5, the total improvement cost increases 10-fold relative to the one-shot complete fulfillment ($u=1$). An effort level of 0.1 increases the cost more than 40,000 times! 

\subsection{Summary}
In this section, we characterize four basic algorithmic approaches to generating CF explanations by their IPF cost property. On the positive side, the optimal algorithm that performs an exhaustive search and its finite search space variant are both IPF stable and thus cost preserving. In addition, if these algorithms are configured to be conservative, in that they only return instances with model prediction over a $p > 0.5$, it is likely that IPF can save total cost by rewarding subjects who take chance. 

On the contrary, algorithms based on gradient descent and randomized search risk increasing the total cost under IPF. The issue can be attributed to the same underlying reason: since these algorithms are not guaranteed to always return the closest CF instance, partial fulfillments in the earlier iterations may be ``cancelled'' by later ones, resulting in increased total costs. In addition, many CF algorithms \citep{mothilal2020dice, russell2019efficient} aim to generate multiple CF instances at the same time in order to provide more diversity and options to the subject. In this case, if the choice made by the subject is not consistent across iterations, the net effect is similar to a randomized search CF algorithm, with higher total improvement cost. 

\section{Empirical Analysis}
\label{sec:empirical}

\begin{table}[!b]
    \centering
    % \vspace{0.1in}
    \caption{Statistics about the dataset and the model. }
    \resizebox{\columnwidth}{!}{
    \begin{tabular}{lrrrr}\toprule
        Dataset & \# Instance & \# Feature (Cat./Num.) & Acc & F1\\\midrule
        Adult Income & 32,561 & 13 (8/5) & 0.84 & 0.66\\
        German Credit & 1,000 & 24 (17/7) & 0.74 & 0.83 \\\bottomrule
    \end{tabular}
    }
    \label{tab:dataset}
\end{table}

\begin{table*}[!bt]
    \centering
    \caption{One sample input instance and two counterfactual explanations for Adult Income (top) and German Credit (bottom). Non-changed feature values are marked with ``-''. Some non-changed features are omitted for presentation. }
    \resizebox{\textwidth}{!}{
    \begin{tabular}{cccccccccccccc}\toprule
        Age & Work Class & Education & Education Num & Marital Status & Occupation & Relationship & Race & Gender & Capital Gain & Capital Loss & Work Hours & Native Country & $m(x)$ \\\midrule
        42 & Self-Employed & HS-Grad & 9 & Married & Craft-Repair & Husband & White & Male & 0 & 0 & 35 & United States & $\leq \$50k$\\
        - & - & Doctorate & 15 & - & - & - & - & - & - & - & - & - & $> \$50k$\\
       - & Local Gov & - & 14 & - & Manager & - & - & - & - & - & - & - & $> \$50k$\\ 
        \midrule
    \end{tabular}
    }
    % \vspace{0.01in}
    \resizebox{\textwidth}{!}{
    \begin{tabular}{cccccccccccccc}\midrule
        Gender & Single & Age & Loan Duration & Purpose of Loan & Loan Amount & Years at Current Home & \# Other Loans & \# Dependents & Has Telephone & No Current Loan & Bank Balance & Housing & $m(x)$ \\\midrule
        Male & True & 42 & 6 & Electronics & 1346 & 4 & 1 & 2 & True & False & 0 & Own & Deny\\
        - & - & - & 24 & Other & - & 3 & 0 & 1 & - & True & - & - & Approve\\
        - & - & - & 12 & Other & - & 1 & 0 & - & - & True & 0-200 & - & Approve\\\bottomrule
    \end{tabular}
    }
    \label{tab:cf-examples}
\end{table*}

\begin{figure*}[!htb]
    \centering
    \includegraphics[width=\textwidth]{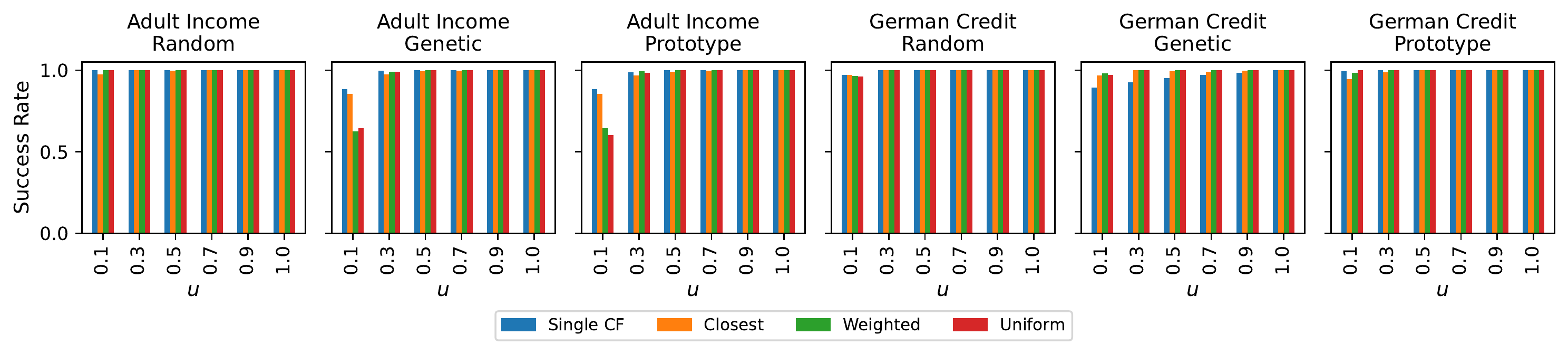}
    \caption{Final success rate of IPF. }
    % \vspace{0.05in}
    \label{fig:success-rate}
\end{figure*}

\subsection{Experiment Setup}
In this section, we empirically study the IPF behaviors of CF algorithms. We use two datasets, Adult Income \citep{adult1996} and German Credit \citep{german1994} The first dataset is about predicting whether the annual salary is above \$50k or not from demographic information collected in the 1994 Census. The second dataset is about predicting whether a person is likely to repay a loan or not from the information about the person's finance and that of the loan.

For each dataset, we use a 80\%/20\% train/test split and apply one-hot encoding to the categorical features and train a random forest classifier as the model. To compute counterfactual explanations, we use the \diceml{} package, which is one of the most popular Python packages for tabular data and non-differentiable classifiers. For all the experiments, we focus on correctly classified negative instances and generate positively predicted CF explanations for them. This scenario is the most common use case of CF explanations as recourses, but our analysis applies to any model input and prediction. Tab.~\ref{tab:dataset} gives summary statistics about the dataset and model performance, and Tab.~\ref{tab:cf-examples} presents some inputs and CF instances.

\diceml{} package searches for diverse CF explanation diversity \citet{mothilal2020dice} in various ways. Since the random forest classifier is not meaningfully differentiable (zero gradient almost everywhere), we study random search -- the default method, genetic search algorithm -- based on the method by \citet{schleich2021geco}, and prototype-guided search with KD tree -- based on the method by \citet{van2021interpretable}. The generated CF explanations are post-processed for sparsity with a feature selection procedure. 

In addition, \diceml{} can generate multiple CF explanations. We study two setups, a single CF explanation (which is still stochastic for random and genetic algorithm search), and 20 CF explanations. In the latter case, we consider three CF selection strategies carried out by the subject: 
\begin{enumerate}[leftmargin=4.5ex]
    \item closest: select the closest CF instance, 
    \item weighted: sample a CF instance from a softmax function (with temperature 1) on the negative distance, and 
    \item uniform: select one CF instance uniformly at random. 
\end{enumerate}
In other words, closest and uniform selections are equivalent to weighted selection with temperature approaching 0 and $\infty$ respectively. We call the setup where only one CF is generated (and hence no selection necessary) as ``single CF.''

For IPF, we use a maximum number of $T=30$ iterations and evaluate effort level $u$ from the set of $\{0.1, 0.3, 0.5, 0.7, 0.9\}$ along with one-shot complete fulfillment $u=1$. At the lowest effort level of $u=0.1$, if the counterfactual explanations were consistent each round, after 30 rounds the input would be to more than 95\% towards the CF ($1 - 0.9^{30}=95.8\%$). We do not employ the $\epsilon$ parameter as none of the algorithms return CF instances exactly on the boundary. 

\subsection{Results}

We first answer the most fundamental question. Can CF algorithms lead to positive predictions in the face of IPF? Fig.~\ref{fig:success-rate} shows the success rate of IPF (up to the maximum number of 30 iterations). Most runs with effort level $u \geq 0.3$ succeed \textit{eventually} without any issue (i.e., getting a positive model prediction). For $u=0.1$, Genetic and Prototype algorithms struggle the most, especially when the final CF is stochastically selected from a diverse set with weighted or uniform distributions.

\begin{figure}[t]
    \centering
    \includegraphics[width=\columnwidth]{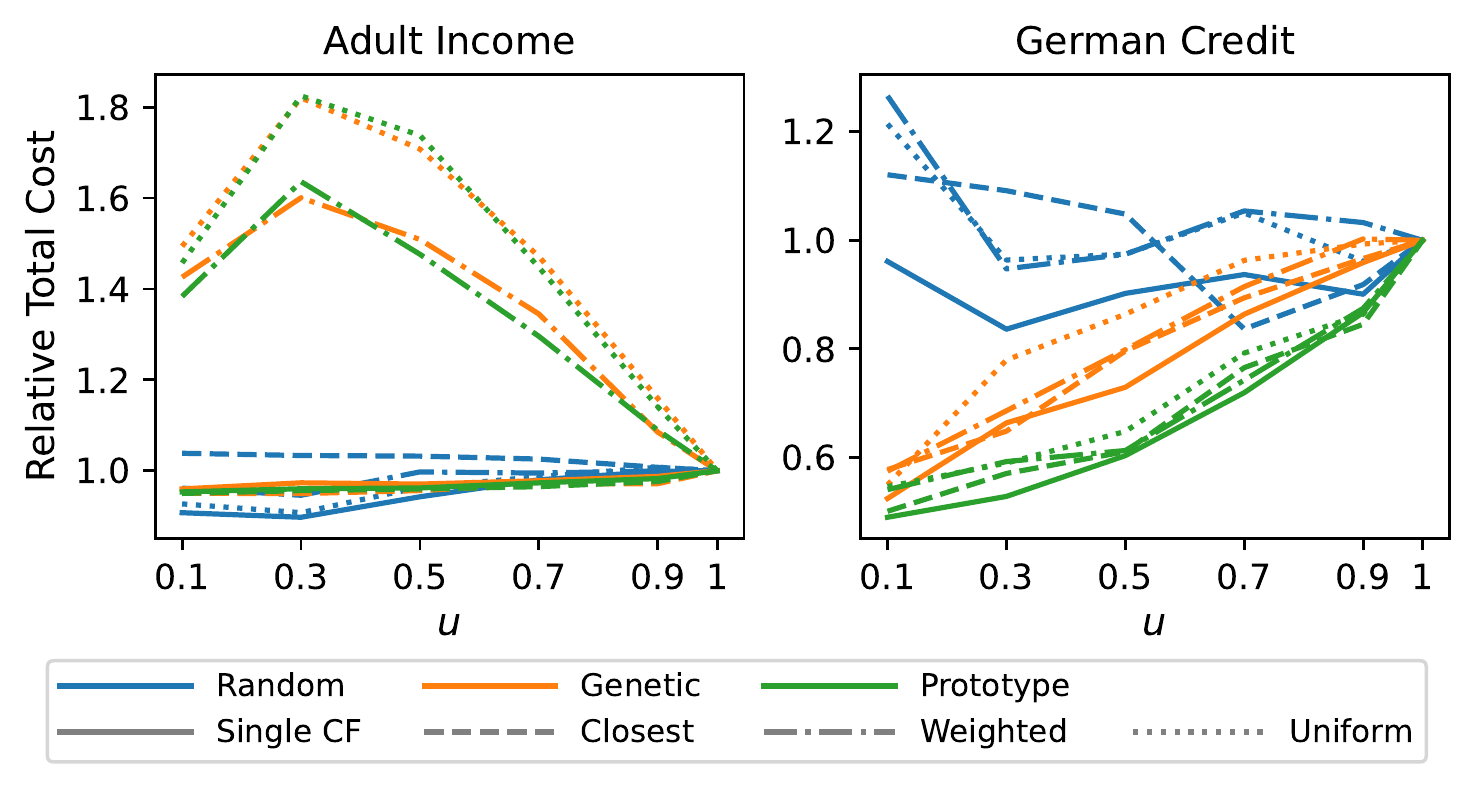}
    \caption{Average total cost at each effort level $u$ relative to that of the one-shot fulfillment $u=1$ for different setups. }
    \label{fig:relative-tc}
\end{figure}

\begin{figure*}[!htb]
    \centering
    \includegraphics[width=\textwidth]{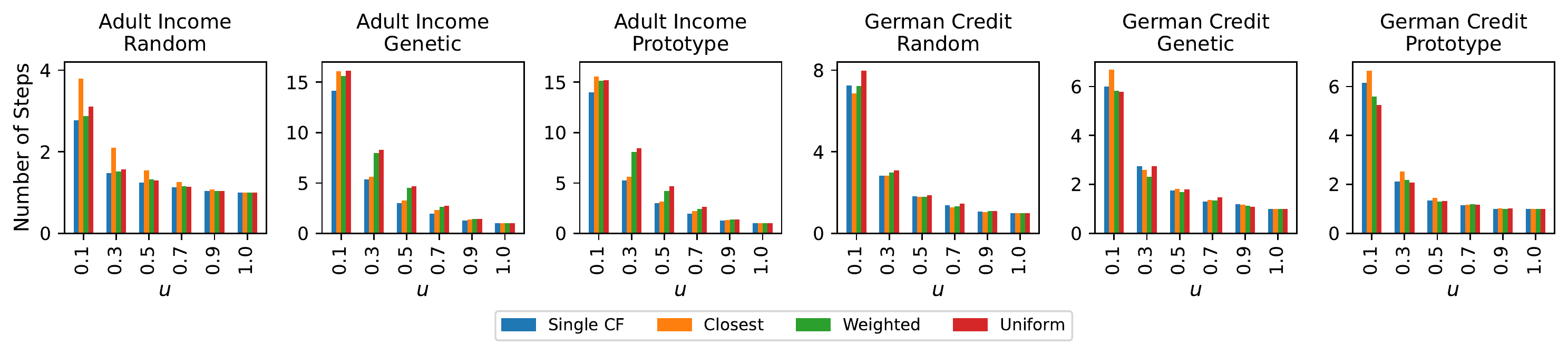}
    \caption{Average number of steps (for successful runs) incurred under different levels of partial fulfillment effort $u$. }
    % \vspace{0.3in}
    \label{fig:num-steps}
\end{figure*}

Focusing on the input instances for which all methods succeed (to ensure a fair comparison), we study the main quantity of interest, total improvement cost under IPF, relative to the one-shot baseline, as plotted in Fig.~\ref{fig:relative-tc}. We observe a variety of behaviors across different setups. The trade-off between cost decrease due to conservatism of CF algorithms (i.e., outputting instances far from the decision boundary) and cost increase due to their non-optimality is best shown on the Adult Income dataset by Genetic and Prototype algorithms, under uniform and weighted selection strategies. In this cases, taking very small steps of $u=0.1$ results in lower total improvement cost than taking medium steps of $u=0.3$ and $0.5$ which may incur a 80\% higher cost, because the small steps in the former helps stop closer to the decision boundary, yet all three choices are inferior to even larger $u$ values, where the inconsistency in different iterations of the search is largely avoided. By comparison, the total improvement cost in other setups of Adult Income are not too sensitive to IPF, although it can have both mildly negative (for Random search with Closest selection) and mildly positive effects (for the remaining setups). 

On German Credit, the Genetic and Prototype algorithms exhibit the opposite effect, showing a monotonic cost decrease with less effort level $u$. One possible reason for this phenomenon is the high-dimensionality of the input space of German Credit vs. Adult Income (24 vs. 13), with more than twice as many categorical features. Thus, it is more likely for \textit{some} categorical features to be changed in German Credit, which, in conjunction with conservative CF explanations, results in smaller total cost under low effort levels. The performance of the Random CF algorithm is similar to that in Adult Income, though with slightly higher variance. 

Fig.~\ref{fig:num-steps} plots the average number of steps until success (for runs that do succeed). As expected, the number increases with decreasing $u$, but the speed of increase varies a lot, with those for Genetic and Prototype algorithms on Adult Income being the largest. Interestingly, the closest selection strategy for the Random search algorithm (orange bar on the leftmost plot) performs markedly worse than the rest, suggesting that such strictly greedy selection from a random sample may be especially suboptimal under IPF. For German Credit, the profiles across different algorithms are largely similar, confirming again that properties of the dataset can be influential in the IPF behaviors of the CF algorithms. 

Overall, the three analyses above demonstrates a variety of behaviors of algorithms under IPF, and hence we advocate for them to be included in a standard suite of evaluations for CF algorithms, as well as considered when developing new CF algorithms. From a human perspective, it may also be necessary to for model users (e.g., banks) to provide explicit guidelines to subjects (e.g., mortgage applicants) to calibrate their expectations on this aspect, which may require new policies to be established on this issue. We provide more discussions in Sec.~\ref{sec:discussion-conclusion}. 

For the rest of this section, we demonstrate how IPF can be incorporated in other, existing aspects of analysis. In particular, as a preliminary investigation, we study the fairness of CF algorithms under IPF. Additional ideas are again discussed in Sec.~\ref{sec:discussion-conclusion}. At a high level, the fairness property requires that different demographic groups (e.g., male vs. female, white vs. other race, etc.) should be treated ``equally,'' with different criteria implementing this notion differently. The criterion that we use is demographic parity \citep{luong2011k, kamishima2011fairness, dwork2012fairness}, one of the simplest and most popular, which basically asserts that for a fair metric (e.g., mortgage application approval), its average value is the same across different demographic groups. Viewed from this angle, we study the fairness of total improvement cost and number of steps under the concept of demographic parity. 

\begin{figure}[t]
    \centering
    \includegraphics[width=\columnwidth]{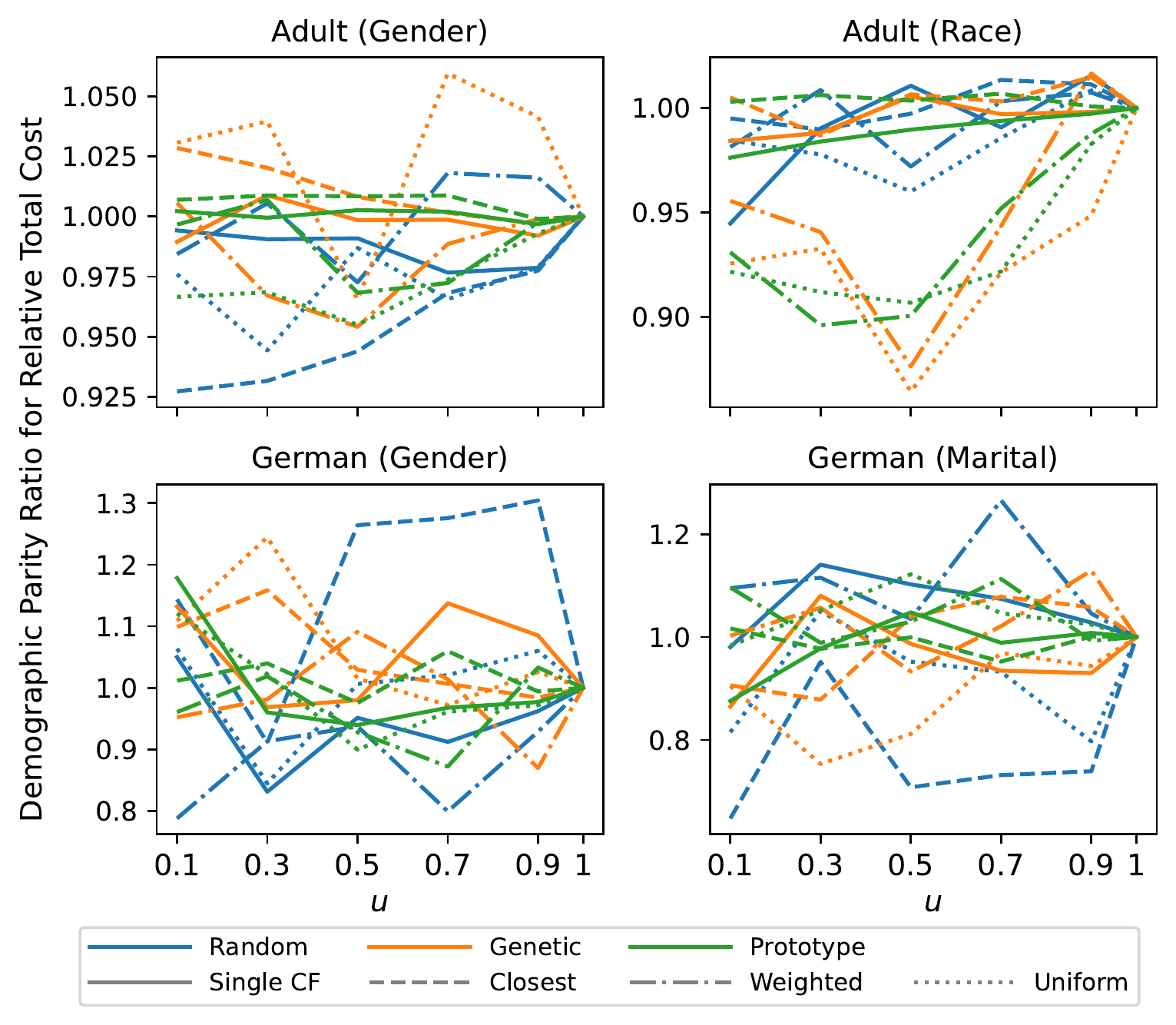}
    \caption{Demographic parity ratio for relative total cost. }
    % \vspace{-200pt}
    % \vspace{0.12in}
    \label{fig:fairness-tc}
\end{figure}

\begin{figure*}
    \centering
    \includegraphics[width=\textwidth]{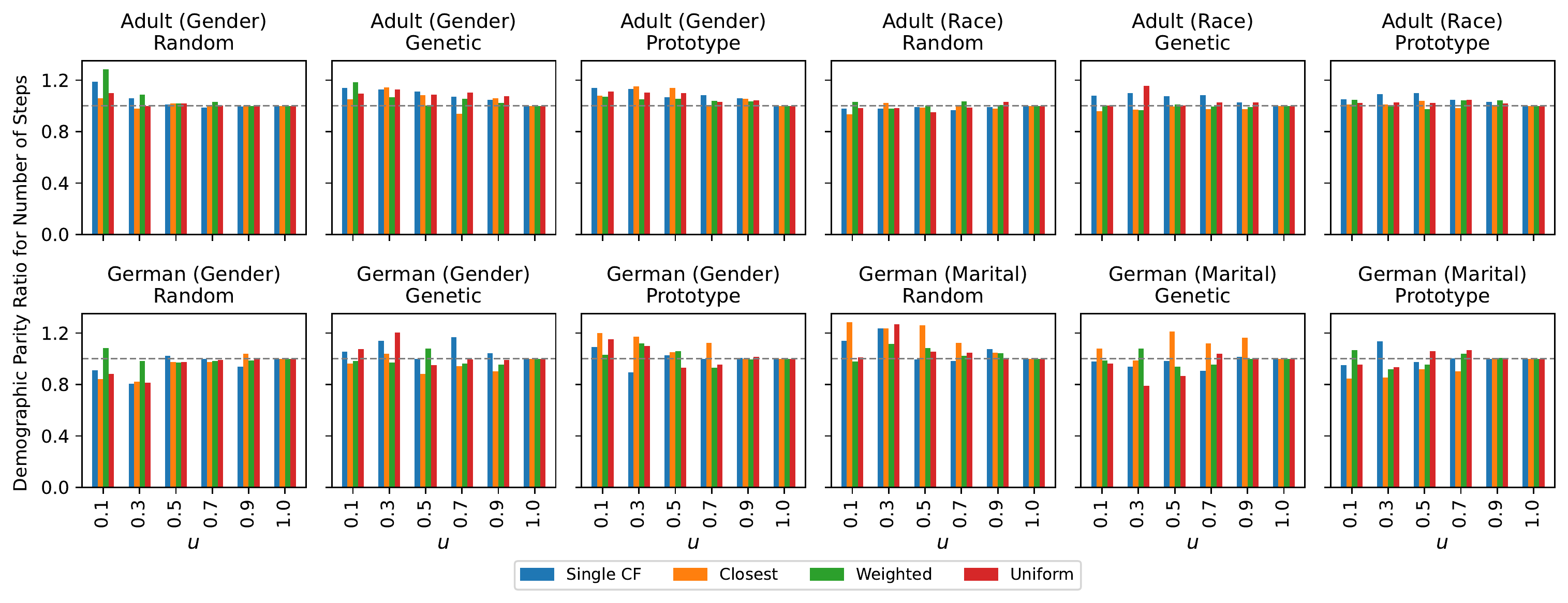}
    \caption{Demographic parity ratio for number of steps. }
    % \vspace{0.3in}
    \label{fig:fairness-num-steps}
\end{figure*}

We consider four demographic group splits in the fairness evaluation, commonly used in the literature \citep{slack2020fooling, slack2021counterfactual, dai2022fairness}. For Adult Income, we study gender with a male/female split, and race with a white/non-white split. For German Credit, we use the same gender split, along with marital status with a married/single split. For each group, we take the second value (e.g., female) as the potentially disadvantaged group and study the ratio of the target of investigation in the disadvantaged group to that in the advantaged group. 

We first study the total improvement cost. Note that we compute the ratio of \textit{relative} total cost (relative to $u=1$), to assess whether IPF \textit{further} exacerbate the fairness issue, on top of what is already observed in the literature for vanilla CF explanations \citep{von2022fairness}, the same target as in Fig.~\ref{fig:relative-tc}. The ratio for these four groups are plotted in Fig.~\ref{fig:fairness-tc}, and while we could not identify any clear and consistent trend, IPF could increase the fairness issue as measured by demographic parity by as much as 30\% for the German Credit model in some settings. 

On the number of steps to achieve success, Fig.~\ref{fig:fairness-num-steps} plots the ratio. The trend is more pronounced. In most setups, the ratio increases as $u$ gets smaller, indicating that IPF has a disproportionately higher impact on the disadvantaged group. Given that the total cost does not demonstrate a clear trend, this means that the per-step improvement cost is \textit{smaller} for the disadvantaged group, which means that the generated CF instances are closer to the queried inputs in the first place. Nonetheless, we leave a definitive verification and further exploration of the implications to future work.

\section{Discussion}
\label{sec:discussion-conclusion}

In this paper, we propose the concept of iterative partial fulfillment, which, to the best of our knowledge, is the first formal study of the situation where the subject of a negative model prediction (e.g., denied mortgage application) does not completely fulfill the given counterfactual (CF) explanation before asking for an updated prediction, for many reasons. First, the subject may intentionally decide to take a chance (e.g., betting that a monthly salary increase of \$800 is enough even though the CF instance requires \$1,000), hoping that a state less qualified than the given CF state is sufficient to secure a positive prediction. Second, the subject may not be able to fully satisfy the CF state (e.g., can only pay down two out of four credit card accounts), especially if given a time limit on the CF validity guarantee (e.g., within the next six months). Furthermore, the subject may misinterpret the CF explanation, such as fulfilling any one of the action items rather than all of them when it is presented as a bullet list. When the partial fulfillment does not result in a positive model prediction, the subject receives a new CF state as part of the rejection and performs an improvement towards the new state. This process repeats until the model prediction is positive, and we call it \textit{iterative partial fulfillment} (IPF). 

Given that virtually all CF algorithms are memoryless (i.e., the CF explanation is generated from only the current input), and most employ local gradient-based or randomized search, it is possible that the CF explanation for a (still negative) partially fulfilled state is different from that of the original input, guiding the subject on a different path of improvement. As a result, the net effect of IPF on the welfare of the subject, most directly measured by final success rate and total improvement cost, could be positive or negative. 

A positive effect could occur when the generated CF instance is conservative, i.e., lying far into the positive prediction region. Such a CF algorithm configuration could be preferred if the model user (e.g., the bank) wants to ensure that the subject is likely to get a positive prediction even if they cannot perfectly follow the CF recommendation. The exact same reasoning allows the subject to engage proactively in partial fulfillment and save on the improvement cost. By contrast, a negative effect could occur when the CF explanation provides different and conflicting advice at different rounds of partial fulfillment. 

In our theoretical analysis, we prove that the optimal cost CF algorithm and its finite search approximation version are guaranteed to not increase total cost under IPF. However, the same could not be said for two popular practical algorithms, gradient ascent and randomized search, both of which worsen subject welfare, sometimes significantly and even potentially unboundedly in theory. 

In our experimental investigation on two datasets, Adult Income and German Credit, totalling 24 CF explainer configurations, we identified both positive and negative effects of IPF, suggesting that IPF is sensitive to properties of the dataset and explainer. As a result, we recommend IPF analysis to be included as part of a standard evaluation suite of CF algorithms. 

For future work, one direction is to consider alternative formulations of IPF. We use a deterministic, fixed-proportion model for continuous features (i.e., for current feature value of $x_d$ and target value $x_d'$, partial fulfillment results in $(1-u)\cdot x_d+u\cdot x_d'$), but this step could be made stochastic by sampling from some distribution centered on $(1-u)\cdot x_d+u\cdot x_d'$, or a fixed-magnitude model could be used where the amount of improvement $\Delta_d$ on each feature is specified. Alternative models on categorical features could also be developed. Last, improvements on some features may be correlated, due to the underlying causal relationships (e.g., change in job title $\rightarrow$ change in salary), so incorporating causal information, potentially in the form of a causal graph, could be explored. Moreover, finding CF algorithms that are stable with respect to more than one IPF notion would be desirable, as different subjects are likely to employ different IPF approaches. 

Additionally, the temporal aspect of the IPF could be studied with more real-world elements. As time goes by in the IPF process, some feature values, such as age, would change in certain manners, which is ignored in the current formulation. Moreover, the very act of querying for a new model prediction may have an impact on some features, such as the bank account balance due to the payment of an application fee, or the credit score due to the bank pulling the credit report, which, at least in the United States, results in a small decrease of the credit score. 

One direction to extend the IPF analysis is to integrate it with other aspects of evaluations. We give a demonstration for the case of fairness, and future work could focus on aspects such as its stability to input perturbations \citep{dominguez2022adversarial, virgolin2023robustness} and model shifts \citep{rawal2020algorithmic}. 

In addition, we focus on IPF analyses of existing CF algorithms, but as a recurring theme of research, the other side of the coin naturally follows: developing new CF algorithms or regularizing existing ones to behave well under IPF scenarios, following analogous works for other CF properties such as robustness \citep{slack2020fooling} and fairness \citep{gupta2019equalizing}. 

Finally, given the diverse and potentially discriminative effects exhibited by IPF, society needs to be better informed and aware of it, especially as some subjects have already been engaging in such behaviors. For example, when the rejection letter of a mortgage application provides some CF explanations as recommendations, the bank may want to, or even be required to, include information about possible outcomes of a re-application with only partial fulfillment. In addition, the application process could allow the applicant to voluntarily reveal their previous applications, so that more stable and consistent CF explanations can be computed, in order to minimize the possibility of conflicting improvement recommendations given to the applicant. All of these changes require not only technical innovations but also policy discussions, for which we hope that this paper serves as a good starting point for such conversations.

\bibliographystyle{ACM-Reference-Format}
\bibliography{big_bib, new}


\begin{thebibliography}{48}

%%% the Web address.

\ifx \showCODEN    \undefined \def \showCODEN     #1{\unskip}     \fi
\ifx \showDOI      \undefined \def \showDOI       #1{#1}\fi
\ifx \showISBNx    \undefined \def \showISBNx     #1{\unskip}     \fi
\ifx \showISBNxiii \undefined \def \showISBNxiii  #1{\unskip}     \fi
\ifx \showISSN     \undefined \def \showISSN      #1{\unskip}     \fi
\ifx \showLCCN     \undefined \def \showLCCN      #1{\unskip}     \fi
\ifx \shownote     \undefined \def \shownote      #1{#1}          \fi
\ifx \showarticletitle \undefined \def \showarticletitle #1{#1}   \fi
\ifx \showURL      \undefined \def \showURL       {\relax}        \fi
% invisible to TeX
\providecommand\bibfield[2]{#2}
\providecommand\bibinfo[2]{#2}
\providecommand\natexlab[1]{#1}
\providecommand\showeprint[2][]{arXiv:#2}

\bibitem[Adebayo et~al\mbox{.}(2022)]%
        {adebayo2021post}
\bibfield{author}{\bibinfo{person}{Julius Adebayo}, \bibinfo{person}{Michael
  Muelly}, \bibinfo{person}{Harold Abelson}, {and} \bibinfo{person}{Been Kim}.}
  \bibinfo{year}{2022}\natexlab{}.
\newblock \showarticletitle{Post Hoc Explanations May Be Ineffective for
  Detecting Unknown Spurious Correlation}. In
  \bibinfo{booktitle}{\emph{International Conference on Learning
  Representations (ICLR)}}.
\newblock


\bibitem[Aryal and Keane(2023)]%
        {aryal2023even}
\bibfield{author}{\bibinfo{person}{Saugat Aryal} {and} \bibinfo{person}{Mark~T
  Keane}.} \bibinfo{year}{2023}\natexlab{}.
\newblock \showarticletitle{Even if Explanations: Prior Work, Desiderata \&
  Benchmarks for Semi-Factual {XAI}}.
\newblock \bibinfo{journal}{\emph{arXiv:2301.11970}} (\bibinfo{year}{2023}).
\newblock


\bibitem[Byrne(2019)]%
        {byrne2019counterfactuals}
\bibfield{author}{\bibinfo{person}{Ruth~MJ Byrne}.}
  \bibinfo{year}{2019}\natexlab{}.
\newblock \showarticletitle{Counterfactuals in Explainable Artificial
  Intelligence ({XAI}): Evidence from Human Reasoning}. In
  \bibinfo{booktitle}{\emph{International Joint Conference on Artificial
  Intelligence (IJCAI)}}. \bibinfo{pages}{6276--6282}.
\newblock


\bibitem[Chen et~al\mbox{.}(2020)]%
        {chen2020linear}
\bibfield{author}{\bibinfo{person}{Yatong Chen}, \bibinfo{person}{Jialu Wang},
  {and} \bibinfo{person}{Yang Liu}.} \bibinfo{year}{2020}\natexlab{}.
\newblock \showarticletitle{Linear Classifiers That Encourage Constructive
  Adaptation}.
\newblock \bibinfo{journal}{\emph{arXiv:2011.00355}} (\bibinfo{year}{2020}).
\newblock


\bibitem[Dai et~al\mbox{.}(2022)]%
        {dai2022fairness}
\bibfield{author}{\bibinfo{person}{Jessica Dai}, \bibinfo{person}{Sohini
  Upadhyay}, \bibinfo{person}{Ulrich Aivodji}, \bibinfo{person}{Stephen~H.
  Bach}, {and} \bibinfo{person}{Himabindu Lakkaraju}.}
  \bibinfo{year}{2022}\natexlab{}.
\newblock \showarticletitle{Fairness via Explanation Quality: Evaluating
  Disparities in the Quality of Post Hoc Explanations}. In
  \bibinfo{booktitle}{\emph{AAAI/ACM Conference on AI, Ethics, and Society
  (AIES)}}. \bibinfo{publisher}{Association for Computing Machinery},
  \bibinfo{pages}{203–214}.
\newblock


\bibitem[De~Lara et~al\mbox{.}(2021)]%
        {de2021transport}
\bibfield{author}{\bibinfo{person}{Lucas De~Lara}, \bibinfo{person}{Alberto
  Gonz{\'a}lez-Sanz}, \bibinfo{person}{Nicholas Asher}, {and}
  \bibinfo{person}{Jean-Michel Loubes}.} \bibinfo{year}{2021}\natexlab{}.
\newblock \showarticletitle{Transport-Based Counterfactual Models}.
\newblock \bibinfo{journal}{\emph{arXiv:2108.13025}} (\bibinfo{year}{2021}).
\newblock


\bibitem[Dominguez-Olmedo et~al\mbox{.}(2022)]%
        {dominguez2022adversarial}
\bibfield{author}{\bibinfo{person}{Ricardo Dominguez-Olmedo},
  \bibinfo{person}{Amir~H Karimi}, {and} \bibinfo{person}{Bernhard
  Sch{\"o}lkopf}.} \bibinfo{year}{2022}\natexlab{}.
\newblock \showarticletitle{On the Adversarial Robustness of Causal Algorithmic
  Recourse}. In \bibinfo{booktitle}{\emph{International Conference on Machine
  Learning (ICML)}}. PMLR, \bibinfo{pages}{5324--5342}.
\newblock


\bibitem[Dua and Graff(1994)]%
        {german1994}
\bibfield{author}{\bibinfo{person}{Dheeru Dua} {and} \bibinfo{person}{Casey
  Graff}.} \bibinfo{year}{1994}\natexlab{}.
\newblock \showarticletitle{{UCI} Statlog ({German} Credit Data) Data Set}.
\newblock \bibinfo{journal}{\emph{UCI Meachine Learning Repository}}
  (\bibinfo{year}{1994}).
\newblock


\bibitem[Dwork et~al\mbox{.}(2012)]%
        {dwork2012fairness}
\bibfield{author}{\bibinfo{person}{Cynthia Dwork}, \bibinfo{person}{Moritz
  Hardt}, \bibinfo{person}{Toniann Pitassi}, \bibinfo{person}{Omer Reingold},
  {and} \bibinfo{person}{Richard Zemel}.} \bibinfo{year}{2012}\natexlab{}.
\newblock \showarticletitle{Fairness Through Awareness}. In
  \bibinfo{booktitle}{\emph{Innovations in Theoretical Computer Science
  (ITCS)}}. \bibinfo{pages}{214--226}.
\newblock


\bibitem[Goodfellow et~al\mbox{.}(2014)]%
        {goodfellow2014explaining}
\bibfield{author}{\bibinfo{person}{Ian Goodfellow}, \bibinfo{person}{Jonathon
  Shlens}, {and} \bibinfo{person}{Christian Szegedy}.}
  \bibinfo{year}{2014}\natexlab{}.
\newblock \showarticletitle{Explaining and Harnessing Adversarial Examples}.
\newblock \bibinfo{journal}{\emph{arXiv:1412.6572}} (\bibinfo{year}{2014}).
\newblock


\bibitem[Guidotti(2022)]%
        {guidotti2022counterfactual}
\bibfield{author}{\bibinfo{person}{Riccardo Guidotti}.}
  \bibinfo{year}{2022}\natexlab{}.
\newblock \showarticletitle{Counterfactual Explanations and How to Find Them:
  Literature Review and Benchmarking}.
\newblock \bibinfo{journal}{\emph{Data Mining and Knowledge Discovery}}
  (\bibinfo{year}{2022}), \bibinfo{pages}{1--55}.
\newblock


\bibitem[Gupta et~al\mbox{.}(2019)]%
        {gupta2019equalizing}
\bibfield{author}{\bibinfo{person}{Vivek Gupta}, \bibinfo{person}{Pegah
  Nokhiz}, \bibinfo{person}{Chitradeep~Dutta Roy}, {and}
  \bibinfo{person}{Suresh Venkatasubramanian}.}
  \bibinfo{year}{2019}\natexlab{}.
\newblock \showarticletitle{Equalizing Recourse Across Groups}.
\newblock \bibinfo{journal}{\emph{arXiv:1909.03166}} (\bibinfo{year}{2019}).
\newblock


\bibitem[Johnson et~al\mbox{.}(2016)]%
        {johnson2016mimic}
\bibfield{author}{\bibinfo{person}{Alistair~EW Johnson}, \bibinfo{person}{Tom~J
  Pollard}, \bibinfo{person}{Lu Shen}, \bibinfo{person}{Li-wei~H Lehman},
  \bibinfo{person}{Mengling Feng}, \bibinfo{person}{Mohammad Ghassemi},
  \bibinfo{person}{Benjamin Moody}, \bibinfo{person}{Peter Szolovits},
  \bibinfo{person}{Leo Anthony~Celi}, {and} \bibinfo{person}{Roger~G Mark}.}
  \bibinfo{year}{2016}\natexlab{}.
\newblock \showarticletitle{{MIMIC-III}, a Freely Accessible Critical Care
  Database}.
\newblock \bibinfo{journal}{\emph{Scientific Data}} \bibinfo{volume}{3},
  \bibinfo{number}{1} (\bibinfo{year}{2016}), \bibinfo{pages}{1--9}.
\newblock


\bibitem[Kamishima et~al\mbox{.}(2011)]%
        {kamishima2011fairness}
\bibfield{author}{\bibinfo{person}{Toshihiro Kamishima},
  \bibinfo{person}{Shotaro Akaho}, {and} \bibinfo{person}{Jun Sakuma}.}
  \bibinfo{year}{2011}\natexlab{}.
\newblock \showarticletitle{Fairness-Aware Learning Through Regularization
  Approach}. In \bibinfo{booktitle}{\emph{IEEE International Conference on Data
  Mining (ICDM) Workshops}}. IEEE, \bibinfo{pages}{643--650}.
\newblock


\bibitem[Kanamori et~al\mbox{.}(2021)]%
        {kanamori2021ordered}
\bibfield{author}{\bibinfo{person}{Kentaro Kanamori}, \bibinfo{person}{Takuya
  Takagi}, \bibinfo{person}{Ken Kobayashi}, \bibinfo{person}{Yuichi Ike},
  \bibinfo{person}{Kento Uemura}, {and} \bibinfo{person}{Hiroki Arimura}.}
  \bibinfo{year}{2021}\natexlab{}.
\newblock \showarticletitle{Ordered Counterfactual Explanation by Mixed-Integer
  Linear Optimization}. In \bibinfo{booktitle}{\emph{AAAI Conference on
  Artificial Intelligence (AAAI)}}, Vol.~\bibinfo{volume}{35}.
  \bibinfo{pages}{11564--11574}.
\newblock


\bibitem[Karimi et~al\mbox{.}(2020)]%
        {karimi2020model}
\bibfield{author}{\bibinfo{person}{Amir-Hossein Karimi},
  \bibinfo{person}{Gilles Barthe}, \bibinfo{person}{Borja Balle}, {and}
  \bibinfo{person}{Isabel Valera}.} \bibinfo{year}{2020}\natexlab{}.
\newblock \showarticletitle{Model-Agnostic Counterfactual Explanations for
  Consequential Decisions}. In \bibinfo{booktitle}{\emph{International
  Conference on Artificial Intelligence and Statistics (AISTATS)}}. PMLR,
  \bibinfo{pages}{895--905}.
\newblock


\bibitem[Karimi et~al\mbox{.}(2021)]%
        {karimi2021algorithmic}
\bibfield{author}{\bibinfo{person}{Amir-Hossein Karimi},
  \bibinfo{person}{Bernhard Sch{\"o}lkopf}, {and} \bibinfo{person}{Isabel
  Valera}.} \bibinfo{year}{2021}\natexlab{}.
\newblock \showarticletitle{Algorithmic Recourse: From Counterfactual
  Explanations to Interventions}. In \bibinfo{booktitle}{\emph{ACM Conference
  on Fairness, Accountability, and Transparency (FAccT)}}.
  \bibinfo{pages}{353--362}.
\newblock


\bibitem[Keane et~al\mbox{.}(2021)]%
        {keane2021if}
\bibfield{author}{\bibinfo{person}{Mark~T Keane}, \bibinfo{person}{Eoin~M
  Kenny}, \bibinfo{person}{Eoin Delaney}, {and} \bibinfo{person}{Barry Smyth}.}
  \bibinfo{year}{2021}\natexlab{}.
\newblock \showarticletitle{If Only We Had Better Counterfactual Explanations:
  Five Key Deficits to Rectify in the Evaluation of Counterfactual {XAI}
  Techniques}. In \bibinfo{booktitle}{\emph{International Joint Conference on
  Artificial Intelligence (IJCAI)}}.
\newblock


\bibitem[Kenny and Keane(2021)]%
        {kenny2021generating}
\bibfield{author}{\bibinfo{person}{Eoin~M Kenny} {and} \bibinfo{person}{Mark~T
  Keane}.} \bibinfo{year}{2021}\natexlab{}.
\newblock \showarticletitle{On Generating Plausible Counterfactual and
  Semi-Factual Explanations for Deep Learning}. In
  \bibinfo{booktitle}{\emph{AAAI Conference on Artificial Intelligence
  (AAAI)}}, Vol.~\bibinfo{volume}{35}. \bibinfo{pages}{11575--11585}.
\newblock


\bibitem[Kohavi and Becker(1996)]%
        {adult1996}
\bibfield{author}{\bibinfo{person}{Ronny Kohavi} {and} \bibinfo{person}{Barry
  Becker}.} \bibinfo{year}{1996}\natexlab{}.
\newblock \showarticletitle{{UCI} {A}dult Data Set}.
\newblock \bibinfo{journal}{\emph{UCI Meachine Learning Repository}}
  (\bibinfo{year}{1996}).
\newblock


\bibitem[Larson et~al\mbox{.}(2016)]%
        {larson2016we}
\bibfield{author}{\bibinfo{person}{Jeff Larson}, \bibinfo{person}{Surya Mattu},
  \bibinfo{person}{Lauren Kirchner}, {and} \bibinfo{person}{Julia Angwin}.}
  \bibinfo{year}{2016}\natexlab{}.
\newblock \showarticletitle{How We Analyzed the {COMPAS} Recidivism Algorithm}.
\newblock \bibinfo{journal}{\emph{ProPublica}} \bibinfo{volume}{9},
  \bibinfo{number}{1} (\bibinfo{year}{2016}), \bibinfo{pages}{3--3}.
\newblock


\bibitem[Liu et~al\mbox{.}(2017)]%
        {liu2016delving}
\bibfield{author}{\bibinfo{person}{Yanpei Liu}, \bibinfo{person}{Xinyun Chen},
  \bibinfo{person}{Chang Liu}, {and} \bibinfo{person}{Dawn Song}.}
  \bibinfo{year}{2017}\natexlab{}.
\newblock \showarticletitle{Delving into Transferable Adversarial Examples and
  Black-Box Attacks}. In \bibinfo{booktitle}{\emph{International Conference on
  Learning Representations (ICLR)}}.
\newblock


\bibitem[Lundberg and Lee(2017)]%
        {lundberg2017unified}
\bibfield{author}{\bibinfo{person}{Scott~M Lundberg} {and}
  \bibinfo{person}{Su-In Lee}.} \bibinfo{year}{2017}\natexlab{}.
\newblock \showarticletitle{A Unified Approach to Interpreting Model
  Predictions}. In \bibinfo{booktitle}{\emph{Advances in Neural Information
  Processing Systems (NIPS)}}. \bibinfo{pages}{4765--4774}.
\newblock


\bibitem[Luong et~al\mbox{.}(2011)]%
        {luong2011k}
\bibfield{author}{\bibinfo{person}{Binh~Thanh Luong},
  \bibinfo{person}{Salvatore Ruggieri}, {and} \bibinfo{person}{Franco Turini}.}
  \bibinfo{year}{2011}\natexlab{}.
\newblock \showarticletitle{{k-NN} as an Implementation of Situation Testing
  for Discrimination Discovery and Prevention}. In
  \bibinfo{booktitle}{\emph{ACM SIGKDD International Conference on Knowledge
  Discovery and Data Mining (KDD)}}. \bibinfo{pages}{502--510}.
\newblock


\bibitem[Maragno et~al\mbox{.}(2023)]%
        {maragno2023finding}
\bibfield{author}{\bibinfo{person}{Donato Maragno}, \bibinfo{person}{Jannis
  Kurtz}, \bibinfo{person}{Tabea~E R{\"o}ber}, \bibinfo{person}{Rob Goedhart},
  \bibinfo{person}{{\c{S}}~Ilker Birbil}, {and} \bibinfo{person}{Dick~den
  Hertog}.} \bibinfo{year}{2023}\natexlab{}.
\newblock \showarticletitle{Finding Regions of Counterfactual Explanations via
  Robust Optimization}.
\newblock \bibinfo{journal}{\emph{arXiv:2301.11113}} (\bibinfo{year}{2023}).
\newblock


\bibitem[Miller(2019)]%
        {miller2019explanation}
\bibfield{author}{\bibinfo{person}{Tim Miller}.}
  \bibinfo{year}{2019}\natexlab{}.
\newblock \showarticletitle{Explanation in Artificial Intelligence: Insights
  from the Social Sciences}.
\newblock \bibinfo{journal}{\emph{Artificial Intelligence}}
  \bibinfo{volume}{267} (\bibinfo{year}{2019}), \bibinfo{pages}{1--38}.
\newblock


\bibitem[Mothilal et~al\mbox{.}(2020)]%
        {mothilal2020dice}
\bibfield{author}{\bibinfo{person}{Ramaravind~K Mothilal},
  \bibinfo{person}{Amit Sharma}, {and} \bibinfo{person}{Chenhao Tan}.}
  \bibinfo{year}{2020}\natexlab{}.
\newblock \showarticletitle{Explaining Machine Learning Classifiers Through
  Diverse Counterfactual Explanations}. In \bibinfo{booktitle}{\emph{ACM
  Conference on Fairness, Accountability, and Transparency (FAT*)}}.
  \bibinfo{pages}{607--617}.
\newblock


\bibitem[Naumann and Ntoutsi(2021)]%
        {naumann2021consequence}
\bibfield{author}{\bibinfo{person}{Philip Naumann} {and}
  \bibinfo{person}{Eirini Ntoutsi}.} \bibinfo{year}{2021}\natexlab{}.
\newblock \showarticletitle{Consequence-Aware Sequential Counterfactual
  Generation}. In \bibinfo{booktitle}{\emph{European Conference on Machine
  Learning and Principles and Practice of Knowledge Discovery in Databases
  (ECML-PKDD)}}. Springer, \bibinfo{pages}{682--698}.
\newblock


\bibitem[Pawelczyk et~al\mbox{.}(2020)]%
        {pawelczyk2020learning}
\bibfield{author}{\bibinfo{person}{Martin Pawelczyk}, \bibinfo{person}{Klaus
  Broelemann}, {and} \bibinfo{person}{Gjergji Kasneci}.}
  \bibinfo{year}{2020}\natexlab{}.
\newblock \showarticletitle{Learning Model-Agnostic Counterfactual Explanations
  for Tabular Data}. In \bibinfo{booktitle}{\emph{The World Wide Web Conference
  (WebConf)}}. \bibinfo{pages}{3126--3132}.
\newblock


\bibitem[Poyiadzi et~al\mbox{.}(2020)]%
        {poyiadzi2020face}
\bibfield{author}{\bibinfo{person}{Rafael Poyiadzi}, \bibinfo{person}{Kacper
  Sokol}, \bibinfo{person}{Raul Santos-Rodriguez}, \bibinfo{person}{Tijl
  De~Bie}, {and} \bibinfo{person}{Peter Flach}.}
  \bibinfo{year}{2020}\natexlab{}.
\newblock \showarticletitle{{FACE}: Feasible and Actionable Counterfactual
  Explanations}. In \bibinfo{booktitle}{\emph{AAAI/ACM Conference on AI,
  Ethics, and Society (AIES)}}. \bibinfo{pages}{344--350}.
\newblock


\bibitem[Ramakrishnan et~al\mbox{.}(2020)]%
        {ramakrishnan2020synthesizing}
\bibfield{author}{\bibinfo{person}{Goutham Ramakrishnan},
  \bibinfo{person}{Yun~Chan Lee}, {and} \bibinfo{person}{Aws Albarghouthi}.}
  \bibinfo{year}{2020}\natexlab{}.
\newblock \showarticletitle{Synthesizing Action Sequences for Modifying Model
  Decisions}. In \bibinfo{booktitle}{\emph{AAAI Conference on Artificial
  Intelligence (AAAI)}}, Vol.~\bibinfo{volume}{34}.
  \bibinfo{pages}{5462--5469}.
\newblock


\bibitem[Rawal et~al\mbox{.}(2020)]%
        {rawal2020algorithmic}
\bibfield{author}{\bibinfo{person}{Kaivalya Rawal}, \bibinfo{person}{Ece
  Kamar}, {and} \bibinfo{person}{Himabindu Lakkaraju}.}
  \bibinfo{year}{2020}\natexlab{}.
\newblock \showarticletitle{Algorithmic Recourse in the Wild: Understanding the
  Impact of Data and Model Shifts}.
\newblock \bibinfo{journal}{\emph{arXiv:2012.11788}} (\bibinfo{year}{2020}).
\newblock


\bibitem[Ribeiro et~al\mbox{.}(2016)]%
        {ribeiro2016should}
\bibfield{author}{\bibinfo{person}{Marco~Tulio Ribeiro},
  \bibinfo{person}{Sameer Singh}, {and} \bibinfo{person}{Carlos Guestrin}.}
  \bibinfo{year}{2016}\natexlab{}.
\newblock \showarticletitle{"{W}hy Should {I} Trust You?" Explaining the
  Predictions of Any Classifier}. In \bibinfo{booktitle}{\emph{ACM SIGKDD
  International Conference on Knowledge Discovery and Data Mining (KDD)}}.
\newblock


\bibitem[Russell(2019)]%
        {russell2019efficient}
\bibfield{author}{\bibinfo{person}{Chris Russell}.}
  \bibinfo{year}{2019}\natexlab{}.
\newblock \showarticletitle{Efficient Search for Diverse Coherent
  Explanations}. In \bibinfo{booktitle}{\emph{ACM Conference on Fairness,
  Accountability, and Transparency (FAT*)}}. \bibinfo{pages}{20--28}.
\newblock


\bibitem[Schleich et~al\mbox{.}(2021)]%
        {schleich2021geco}
\bibfield{author}{\bibinfo{person}{Maximilian Schleich},
  \bibinfo{person}{Zixuan Geng}, \bibinfo{person}{Yihong Zhang}, {and}
  \bibinfo{person}{Dan Suciu}.} \bibinfo{year}{2021}\natexlab{}.
\newblock \showarticletitle{{GeCo}: Quality Counterfactual Explanations in Real
  Time}.
\newblock \bibinfo{journal}{\emph{Proceedings of the VLDB Endowment}}
  \bibinfo{volume}{14}, \bibinfo{number}{9} (\bibinfo{date}{May}
  \bibinfo{year}{2021}), \bibinfo{pages}{1681–1693}.
\newblock
\showISSN{2150-8097}


\bibitem[Simonyan et~al\mbox{.}(2013)]%
        {simonyan2013deep}
\bibfield{author}{\bibinfo{person}{Karen Simonyan}, \bibinfo{person}{Andrea
  Vedaldi}, {and} \bibinfo{person}{Andrew Zisserman}.}
  \bibinfo{year}{2013}\natexlab{}.
\newblock \showarticletitle{Deep Inside Convolutional Networks: Visualising
  Image Classification Models and Saliency Maps}.
\newblock \bibinfo{journal}{\emph{arXiv:1312.6034}} (\bibinfo{year}{2013}).
\newblock


\bibitem[Slack et~al\mbox{.}(2021)]%
        {slack2021counterfactual}
\bibfield{author}{\bibinfo{person}{Dylan Slack}, \bibinfo{person}{Anna
  Hilgard}, \bibinfo{person}{Himabindu Lakkaraju}, {and}
  \bibinfo{person}{Sameer Singh}.} \bibinfo{year}{2021}\natexlab{}.
\newblock \showarticletitle{Counterfactual Explanations Can Be Manipulated}. In
  \bibinfo{booktitle}{\emph{Advances in Neural Information Processing Systems
  (NeurIPS)}}, Vol.~\bibinfo{volume}{34}. \bibinfo{pages}{62--75}.
\newblock


\bibitem[Slack et~al\mbox{.}(2020)]%
        {slack2020fooling}
\bibfield{author}{\bibinfo{person}{Dylan Slack}, \bibinfo{person}{Sophie
  Hilgard}, \bibinfo{person}{Emily Jia}, \bibinfo{person}{Sameer Singh}, {and}
  \bibinfo{person}{Himabindu Lakkaraju}.} \bibinfo{year}{2020}\natexlab{}.
\newblock \showarticletitle{Fooling {LIME} and {SHAP}: Adversarial Attacks on
  Post Hoc Explanation Methods}. In \bibinfo{booktitle}{\emph{AAAI/ACM
  Conference on AI, Ethics, and Society (AIES)}}.
  \bibinfo{publisher}{Association for Computing Machinery},
  \bibinfo{pages}{180--186}.
\newblock


\bibitem[Upadhyay et~al\mbox{.}(2021)]%
        {upadhyay2021towards}
\bibfield{author}{\bibinfo{person}{Sohini Upadhyay}, \bibinfo{person}{Shalmali
  Joshi}, {and} \bibinfo{person}{Himabindu Lakkaraju}.}
  \bibinfo{year}{2021}\natexlab{}.
\newblock \showarticletitle{Towards Robust and Reliable Algorithmic Recourse}.
  In \bibinfo{booktitle}{\emph{Advances in Neural Information Processing
  Systems (NeurIPS)}}, Vol.~\bibinfo{volume}{34}.
  \bibinfo{pages}{16926--16937}.
\newblock


\bibitem[Ustun et~al\mbox{.}(2019)]%
        {ustun2019actionable}
\bibfield{author}{\bibinfo{person}{Berk Ustun}, \bibinfo{person}{Alexander
  Spangher}, {and} \bibinfo{person}{Yang Liu}.}
  \bibinfo{year}{2019}\natexlab{}.
\newblock \showarticletitle{Actionable Recourse in Linear Classification}. In
  \bibinfo{booktitle}{\emph{ACM Conference on Fairness, Accountability, and
  Transparency (FAT*)}}. \bibinfo{pages}{10--19}.
\newblock


\bibitem[Van~Looveren and Klaise(2021)]%
        {van2021interpretable}
\bibfield{author}{\bibinfo{person}{Arnaud Van~Looveren} {and}
  \bibinfo{person}{Janis Klaise}.} \bibinfo{year}{2021}\natexlab{}.
\newblock \showarticletitle{Interpretable Counterfactual Explanations Guided by
  Prototypes}. In \bibinfo{booktitle}{\emph{European Conference on Machine
  Learning and Principles and Practice of Knowledge Discovery in Databases
  (ECML-PKDD)}}. Springer, \bibinfo{pages}{650--665}.
\newblock


\bibitem[Verma et~al\mbox{.}(2020)]%
        {verma2020counterfactual}
\bibfield{author}{\bibinfo{person}{Sahil Verma}, \bibinfo{person}{John
  Dickerson}, {and} \bibinfo{person}{Keegan Hines}.}
  \bibinfo{year}{2020}\natexlab{}.
\newblock \showarticletitle{Counterfactual Explanations for Machine Learning: A
  Review}.
\newblock \bibinfo{journal}{\emph{arXiv:2010.10596}} (\bibinfo{year}{2020}).
\newblock


\bibitem[Verma et~al\mbox{.}(2022)]%
        {verma2022amortized}
\bibfield{author}{\bibinfo{person}{Sahil Verma}, \bibinfo{person}{Keegan
  Hines}, {and} \bibinfo{person}{John~P Dickerson}.}
  \bibinfo{year}{2022}\natexlab{}.
\newblock \showarticletitle{Amortized Generation of Sequential Algorithmic
  Recourses for Black-Box Models}. In \bibinfo{booktitle}{\emph{AAAI Conference
  on Artificial Intelligence (AAAI)}}, Vol.~\bibinfo{volume}{36}.
  \bibinfo{pages}{8512--8519}.
\newblock


\bibitem[Virgolin and Fracaros(2023)]%
        {virgolin2023robustness}
\bibfield{author}{\bibinfo{person}{Marco Virgolin} {and}
  \bibinfo{person}{Saverio Fracaros}.} \bibinfo{year}{2023}\natexlab{}.
\newblock \showarticletitle{On the Robustness of Sparse Counterfactual
  Explanations to Adverse Perturbations}.
\newblock \bibinfo{journal}{\emph{Artificial Intelligence}}
  \bibinfo{volume}{316} (\bibinfo{year}{2023}).
\newblock


\bibitem[von K{\"u}gelgen et~al\mbox{.}(2022)]%
        {von2022fairness}
\bibfield{author}{\bibinfo{person}{Julius von K{\"u}gelgen},
  \bibinfo{person}{Amir-Hossein Karimi}, \bibinfo{person}{Umang Bhatt},
  \bibinfo{person}{Isabel Valera}, \bibinfo{person}{Adrian Weller}, {and}
  \bibinfo{person}{Bernhard Sch{\"o}lkopf}.} \bibinfo{year}{2022}\natexlab{}.
\newblock \showarticletitle{On the Fairness of Causal Algorithmic Recourse}. In
  \bibinfo{booktitle}{\emph{AAAI Conference on Artificial Intelligence
  (AAAI)}}, Vol.~\bibinfo{volume}{36}. \bibinfo{pages}{9584--9594}.
\newblock


\bibitem[Wachter et~al\mbox{.}(2017)]%
        {wachter2017counterfactual}
\bibfield{author}{\bibinfo{person}{Sandra Wachter}, \bibinfo{person}{Brent
  Mittelstadt}, {and} \bibinfo{person}{Chris Russell}.}
  \bibinfo{year}{2017}\natexlab{}.
\newblock \showarticletitle{Counterfactual Explanations Without Opening the
  Black Box: Automated Decisions and The {GDPR}}.
\newblock \bibinfo{journal}{\emph{Harvard Journal of Law \& Technology}}
  \bibinfo{volume}{31} (\bibinfo{year}{2017}), \bibinfo{pages}{841}.
\newblock


\bibitem[Yang and Kim(2019)]%
        {yang2019benchmarking}
\bibfield{author}{\bibinfo{person}{Mengjiao Yang} {and} \bibinfo{person}{Been
  Kim}.} \bibinfo{year}{2019}\natexlab{}.
\newblock \showarticletitle{Benchmarking Attribution Methods with Relative
  Feature Importance}.
\newblock \bibinfo{journal}{\emph{arXiv:1907.09701}} (\bibinfo{year}{2019}).
\newblock


\bibitem[Zhou et~al\mbox{.}(2022)]%
        {zhou2021feature}
\bibfield{author}{\bibinfo{person}{Yilun Zhou}, \bibinfo{person}{Serena Booth},
  \bibinfo{person}{Marco~Tulio Ribeiro}, {and} \bibinfo{person}{Julie Shah}.}
  \bibinfo{year}{2022}\natexlab{}.
\newblock \showarticletitle{Do Feature Attribution Methods Correctly Attribute
  Features?}. In \bibinfo{booktitle}{\emph{AAAI Conference on Artificial
  Intelligence (AAAI)}}.
\newblock


\end{thebibliography}

\end{document}